\documentclass{article}

\PassOptionsToPackage{numbers,compress}{natbib}
\usepackage[main,preprint]{neurips_2026}

\usepackage[utf8]{inputenc}
\usepackage[T1]{fontenc}
\usepackage{microtype}
\usepackage{graphicx}
\usepackage{booktabs}
\usepackage{multirow}
\usepackage{tabularx}
\usepackage{float}
\usepackage{placeins}
\usepackage{needspace}
\usepackage{algorithm}
\usepackage{algpseudocode}
\usepackage{amsmath}
\usepackage{mathtools}
\usepackage{amssymb}
\usepackage{amsfonts}
\usepackage{amsthm}
\usepackage{bm}
\usepackage{nicefrac}
\usepackage{xcolor}
\usepackage{enumitem}
\usepackage{tikz}
\usetikzlibrary{arrows.meta,positioning}
\usepackage{url}
\usepackage{hyperref}
\hypersetup{
  colorlinks=true,
  linkcolor=blue!45!black,
  citecolor=blue!45!black,
  urlcolor=blue!45!black
}
\graphicspath{{workspace/figs/}}

\newcommand{\R}{\mathbb{R}}
\newcommand{\E}{\mathbb{E}}
\newcommand{\KL}{D_{\mathrm{KL}}}

\newcommand{\Tr}{\operatorname{Tr}}
\newcommand{\diag}{\operatorname{diag}}

\newcommand{\argmin}{\operatorname*{arg\,min}}

\newcommand{\ntxt}[1]{\text{\normalfont #1}}
\newcommand{\GL}{\mathrm{GL}}
\newcommand{\Orth}{\mathrm{O}}

\newcommand{\tp}{{}^{\!\top}}
\newcommand{\Price}{\rho}
\newcommand{\hPrice}{\hat{\rho}}
\newcommand{\Cost}{\kappa}

\newcommand{\eps}{\varepsilon}

\newcommand{\vw}{\mathbf{w}}
\newcommand{\vx}{\mathbf{x}}

\newcommand{\vz}{\mathbf{z}}

\newcommand{\vg}{\mathbf{g}}
\newcommand{\vh}{\mathbf{h}}
\newcommand{\vc}{\mathbf{c}}
\newcommand{\vdelta}{\boldsymbol{\delta}}

\newcommand{\rmA}{\mathbf{A}}

\newcommand{\rmC}{\mathbf{C}}
\newcommand{\rmD}{\mathbf{D}}
\newcommand{\rmE}{\mathbf{E}}
\newcommand{\rmH}{\mathbf{H}}
\newcommand{\rmI}{\mathbf{I}}

\newcommand{\rmM}{\mathbf{M}}
\newcommand{\rmN}{\mathbf{N}}
\newcommand{\rmP}{\mathbf{P}}
\newcommand{\rmQ}{\mathbf{Q}}
\newcommand{\rmR}{\mathbf{R}}
\newcommand{\rmS}{\mathbf{S}}

\newcommand{\rmU}{\mathbf{U}}
\newcommand{\rmV}{\mathbf{V}}
\newcommand{\rmW}{\mathbf{W}}
\newcommand{\rmX}{\mathbf{X}}

\newcommand{\Dmat}{\boldsymbol{\Delta}}
\newcommand{\Sigmamat}{\boldsymbol{\Sigma}}

\newcommand{\W}{\rmW}
\newcommand{\X}{\rmX}
\newcommand{\hW}{\hat{\rmW}}
\newcommand{\hHz}{\hat{\rmH}}
\newcommand{\Pmat}{\rmP}
\newcommand{\Qmat}{\rmQ}
\newcommand{\Smat}{\rmS}
\newcommand{\Rerr}{\rmR}

\newcommand{\Cerr}{\Sigmamat}
\newcommand{\hCerr}{\hat{\Sigmamat}}

\newcommand{\Jcal}{\mathcal{J}}
\newcommand{\Acal}{\mathcal{A}}
\newcommand{\Dcal}{\mathcal{D}}

\newcommand{\Lref}{\mathcal{L}}

\newcommand{\pizero}{\pi_0}
\newcommand{\Nout}{\rmN}

\newcommand{\Cx}{\rmC}
\newcommand{\Cxl}{\rmC_l}

\newcommand{\Hz}{\rmH}

\newcommand{\I}{\rmI}
\newcommand{\tDelta}{\tilde{\Dmat}}
\newcommand{\tW}{\tilde{\rmW}}
\newcommand{\tx}{\tilde{\vx}}
\newcommand{\cgeom}{\bar c_{\ntxt{geo}}}

\newcommand{\Levels}{\mathsf{L}}
\DeclarePairedDelimiter{\norm}{\lVert}{\rVert}
\DeclarePairedDelimiter{\abs}{\lvert}{\rvert}

\providecommand{\Acal}{\mathcal{A}}

\providecommand{\Jcal}{\mathcal{J}}

\providecommand{\eps}{\varepsilon}

\usepackage[capitalize,noabbrev]{cleveref}

\crefname{section}{Sec.}{Secs.}
\Crefname{section}{Sec.}{Secs.}
\crefname{subsection}{Sec.}{Secs.}
\Crefname{subsection}{Sec.}{Secs.}
\crefname{subsubsection}{Sec.}{Secs.}
\Crefname{subsubsection}{Sec.}{Secs.}
\crefname{theorem}{Thm.}{Thms.}
\Crefname{theorem}{Thm.}{Thms.}
\crefname{lemma}{Lem.}{Lems.}
\Crefname{lemma}{Lem.}{Lems.}
\crefname{proposition}{Prop.}{Props.}
\Crefname{proposition}{Prop.}{Props.}
\crefname{corollary}{Cor.}{Cors.}
\Crefname{corollary}{Cor.}{Cors.}
\crefname{definition}{Def.}{Defs.}
\Crefname{definition}{Def.}{Defs.}
\crefname{assumption}{Assump.}{Assumps.}
\Crefname{assumption}{Assump.}{Assumps.}
\crefname{remark}{Rem.}{Rems.}
\Crefname{remark}{Rem.}{Rems.}
\crefname{algorithm}{Alg.}{Algs.}
\Crefname{algorithm}{Alg.}{Algs.}
\crefname{appendix}{Appx.}{Appx.}
\Crefname{appendix}{Appx.}{Appx.}

\let\oldappendix\appendix
\renewcommand{\appendix}{%
	\oldappendix
	\crefalias{section}{appendix}%
	\crefalias{subsection}{appendix}%
	\crefalias{subsubsection}{appendix}%
}

\newtheoremstyle{paperplain}
  {3pt}{3pt}
  {\fontfamily{LibertinusSerif-TLF}\selectfont\itshape}
  {}
  {\fontfamily{LibertinusSerif-TLF}\selectfont\bfseries}
  {.}
  {5pt plus 1pt minus 1pt}
  {}

\usepackage{aliascnt}
\theoremstyle{paperplain}
\newtheorem{theorem}{Theorem}[section]
\newaliascnt{proposition}{theorem}

\aliascntresetthe{proposition}
\newaliascnt{lemma}{theorem}

\aliascntresetthe{lemma}
\newaliascnt{corollary}{theorem}

\aliascntresetthe{corollary}
\theoremstyle{definition}
\newaliascnt{definition}{theorem}

\aliascntresetthe{definition}
\newaliascnt{assumption}{theorem}

\aliascntresetthe{assumption}
\newaliascnt{remark}{theorem}
\newtheorem{remark}[remark]{Remark}
\aliascntresetthe{remark}

\title{Predicting Quantization Price for Selecting PTQ Configurations Before Deployment}

\author{
  Junbin Qiu \quad
  Jian Mu \quad
  Weitong Zhang \quad
  Yao Shu
}

\makeatletter
\renewcommand{\@noticestring}{%
  Preprint. Correspondence to Yao Shu \texttt{<yaoshu@hkust-gz.edu.cn>}.%
}
\makeatother

\begin{document}

\maketitle

\begin{abstract}
Weight-space post-training quantization (PTQ) must choose finite formats, granularities, quantizer families, transformations, and bits before the completed quantized model reveals its output-distribution drift.
Existing PTQ methods predict important pieces of this degradation, including reconstruction error, Hessian sensitivity, transformation effects, and downstream loss, but these pieces are usually scored after fixing the quantization geometry or inside separate configuration families.
We formulate weight-space PTQ as pre-deployment configuration selection using priced layer-output error.
Each admissible layer configuration is treated as an error generator with a deployment cost, which induces a layer-output error covariance $\Cerr_l(\alpha_l)$, and the full-precision model prices that covariance by downstream curvature, $\hPrice_l(\alpha_l)=\frac12\Tr(\hHz_l\hCerr_l(\alpha_l))$.
The price follows from full-precision-to-quantized forward KL, whose first-order term cancels at the reference model.
It turns reconstruction and diagonal scores into reduced proxies that drop price factors, while finite formats, codebooks, granularities, and equivalent transformations become comparable candidates through the covariances they induce and the costs they pay.
A trace reduction then yields a calibration-time price table and a budgeted price-guided selector, making fixed-geometry bit allocation a special case rather than the organizing problem.
\end{abstract}

\section{Introduction}

Weight-space PTQ is chosen before its final consequence is visible.
A deployment stack must decide, layer by layer, which finite format, granularity, quantizer family, pre-quantization transformation, and bitwidth to use, while the completed weight-quantized model reveals its output-distribution drift only after all of these choices have been composed.
Exhaustively building every candidate model is usually impossible, yet reducing the decision to mixed precision after fixing the geometry misses a central degree of freedom because the configuration itself determines what layer-output error will be injected.
The practical question is therefore not only how many bits to allocate, but how to rank admissible weight-space PTQ configurations before deployment trials.
The same calibration pass must score choices that have not yet been deployed.

Prior work already shows why this question cannot be answered by one local proxy alone.
Reconstruction and second-order PTQ methods such as OBQ, OBC, GPTQ, AdaRound, BRECQ, QuIP, and QuIP\# reduce or compensate local weight error \citep{nagel2020adaround,li2021brecq,frantar2022gptq,frantar2023obc,chee2023quip,tseng2024quipsharp}.
HAWQ-style methods, HIGGS, and Q-Palette use Hessian, sensitivity, linear-response, or rate-distortion signals for precision or level allocation \citep{dong2019hawq,dong2019hawqv2,malinovskii2024linearity,lee2025qpalette}.
SmoothQuant, AWQ, QuaRot, SpinQuant, OmniQuant, AffineQuant, and FlatQuant show that scaling, rotations, incoherence, outlier handling, and learned transforms change quantization behavior \citep{xiao2023smoothquant,lin2024awq,ashkboos2024quarot,liu2024spinquant,shao2023omniquant,ma2024affinequant,sun2024flatquant}.
GuidedQuant and YAQA bring downstream loss or output-distribution information into quantization objectives \citep{guidedquant2025,tseng2025mpadaptive}.
These results predict important pieces of PTQ degradation.
The missing object is a common pre-deployment price that compares bits, granularities, codebooks, transformations, and hardware costs as configurations of the same decision.

This paper makes that price the organizing object.
A layer configuration is an admissible weight-space PTQ choice supplied by the deployment stack, and it may change the candidate coordinate, choose a finite-format replacement, and carry a backend cost.
\cref{sec:problem-setup} turns this into a constrained selection problem over configurations under a deployment budget, with the full-precision-to-quantized forward KL as the target drift.
That ideal drift is not available before building the completed quantized model, so the selection problem forces a calibration-time substitute, namely the layer-output error covariance induced by a candidate and priced by the downstream curvature of the full-precision reference.
This covariance, together with cost, is the interface between broad PTQ mechanisms and the selector.

\cref{sec:priced-error} derives this substitute from the forward KL rather than postulating another reconstruction score.
At the full-precision reference, score cancellation removes the linear term, so the first retained local contribution is a quadratic price of candidate-induced layer-output error.
The candidate supplies the covariance of that error, while the reference model supplies the output-side curvature that makes some error directions more costly than others.
This immediately explains the role of familiar proxies because reconstruction drops the output curvature, and diagonal input scores also drop cross-channel input covariance.
Such proxies can be useful reduced scores, but they preserve configuration rankings only when the factors they drop are effectively constant over the candidates being compared.
The price therefore remains a configuration criterion rather than a reconstruction metric.

The same price also explains how PTQ configuration families enter the selector.
Finite formats, bits, granularities, vector codebooks, and palettes change the weight-perturbation moment being priced.
Equivalent transformations change the coordinate in which that perturbation is generated and the input covariance it acts on, while leaving the reference-side output metric fixed.
Under the isotropic trace reduction used for large models, each candidate can be scored by two candidate-side scalars and a cached curvature trace, then selected by a budgeted search over admissible costs.
Fixed-geometry bit allocation is therefore recovered only after the transformation, quantizer family, and granularity have been held fixed.
Before that point, PTQ is a broader configuration-selection problem over deployable choices, not an abstract continuous bit variable.

\paragraph{Contributions.}
The contribution centers on a price-prediction criterion for the configuration decision, using forward-KL and curvature signals to compare admissible PTQ choices across existing quantizer families.
\textbf{First}, the paper identifies the object selected by weight-space PTQ as an admissible layer configuration that induces a layer-output error covariance and has a deployment cost.
\textbf{Second}, it shows that full-precision-to-quantized forward KL assigns this covariance a quadratic quantization price, making reconstruction and diagonal scores interpretable as reduced prices.
\textbf{Third}, it connects common PTQ families to the same price factors by showing that finite formats change perturbation moments, transformations change the candidate coordinate and input metric, and costs constrain deployability.
\textbf{Fourth}, it turns these prices into a budgeted pre-deployment search rule that optimizes predicted quantization price directly, with mechanism evidence and scope boundaries organized around the same prediction chain.

\section{Pre-Deployment Configuration Selection in Weight-Space PTQ}
\label{sec:problem-setup}

\subsection{Reference Model, Candidate Coordinates, and Target Drift}
\label{sec:predeployment-selection}

Weight-space PTQ starts from a full-precision reference model and a calibration distribution $\Dcal$, before any weight-quantized model has been deployed.
Let the model have $L$ weight-bearing layers with reference weights $\W_1,\ldots,\W_L$.
For a calibration input $\vx$, let $\vx_l$ be the reference input to layer $l$ and $\vz_l=\W_l\vx_l$ be its reference output.
For each layer $l$, the deployment stack supplies an admissible set $\Acal_l$.
A layer configuration $\alpha_l\in\Acal_l$ is the pre-deployment object that specifies bitwidth, granularity, quantizer family, hardware format, pre-quantization transformation, or a supported combination of these choices.
Candidate-dependent layer objects carry this local choice as an argument.
When $\alpha_l$ changes coordinate before quantization, write the candidate coordinate as $\vx_l(\alpha_l)$ and $\W_l(\alpha_l)$, which should satisfy the configuration invariance:
\begin{equation}
\W_l(\alpha_l)\vx_l(\alpha_l)=\W_l\vx_l \ ,
\end{equation}
where the ordinary weight quantization (without any coordinate change) is $\vx_l(\alpha_l)=\vx_l$, $\W_l(\alpha_l)=\W_l$.
For a calibration batch, $\X_l(\alpha_l)$ collects the columns $\vx_{l,i}(\alpha_l)$, and $\X_l$ denotes the reference-coordinate batch.

The same configuration also carries a deployment cost $\Cost_l(\alpha_l)$, and the deployment budget is $B$.
A full configuration $\alpha=(\alpha_1,\ldots,\alpha_L)$ applies one layer choice to each layer.
Applying it produces a weight-quantized model with output distribution $\pi_\alpha(\cdot\mid\vx)$, while the full-precision reference distribution is $\pizero(\cdot\mid\vx)$.
The quantity to be controlled is the final output-distribution drift:
\begin{equation}
\label{eq:target-drift}
    \Jcal(\alpha)\triangleq
    \E_{\vx\sim\Dcal}\left[
    \KL\big(\pizero(\cdot\mid\vx)\,\Vert\,\pi_\alpha(\cdot\mid\vx)\big)\right]\ .
\end{equation}

If this drift were known for every feasible configuration, PTQ would be the constrained selection problem
\begin{equation}
\label{eq:ideal-configuration-selection}
    \min_{\alpha\in\Acal_1\times\cdots\times\Acal_L}\Jcal(\alpha)
    \quad \mathrm{s.t.}\quad
    \sum_{l=1}^L \Cost_l(\alpha_l)\le B \ .
\end{equation}

This constrained program names the deployment decision, but evaluating $\Jcal(\alpha)$ already requires the selected transformations, granularities, quantizer families, and bits to have produced a complete model.
PTQ therefore needs a calibration-time rule for ranking configurations before this final model-level trial.
The remaining preliminary step is to unpack $\Acal_l$ only through objects available before that trial, including the candidate coordinate, the finite-format replacement, and the deployment cost.

\subsection{Admissible PTQ Configuration Families and Costs}
\label{sec:ptq-configuration-families}

The admissible set $\Acal_l$ is deliberately broad, but the price theory will only ask what a candidate can provide before deployment.
The preceding coordinate convention lets us describe those candidates without reopening the full layerwise notation.
Fix one layer, write $\alpha$ for its local choice, and suppress the layer index.
A candidate maps the reference calibration pair $(\W,\X)$ to a floating-point coordinate $(\W(\alpha),\X(\alpha))$ that still satisfies $\W(\alpha)\X(\alpha)=\W\X$, and also supplies the finite-format replacement $\hW(\alpha)$ and the deployment cost $\Cost(\alpha)$.
The coordinate equality means that pre-quantization transformations are equivalent before rounding, while $\hW(\alpha)$ is where finite-format error enters.
The main families differ in how they choose the allowed values, how they change the coordinate, and what deployment cost they pay.

\textbf{Finite formats choose the allowed values.}
A scalar group quantizer chooses a code set $\mathcal C_g$, a scale $s_g$, and a bitwidth $b_g$, then projects each group of weights onto those values:
\begin{equation}
    \hW_g(\alpha)
    =
    s_g\,\Pi_{\mathcal C_g}(\W_g(\alpha)/s_g) \ ,
    \qquad
    \delta_g=\frac{2r_g}{2^{b_g}-1} \ ,
\end{equation}
where $\Pi_{\mathcal C_g}$ is the projection onto the code set $\mathcal C_g$.

The goal is direct compression because smaller $b_g$ and coarser sharing reduce storage and kernel cost, but increase the local weight perturbation $\hW_g(\alpha)-\W_g(\alpha)$.
Per-tensor, per-channel, group-wise, and block-wise quantizers differ by which weights share $s_g,r_g,b_g$ \citep{dong2019hawq,dong2019hawqv2}.
Hardware floating-point and microscaling formats change the same choice through exponent sharing and admissible scales \citep{sharify2024microscaling,zhang2026mxfpbenchmark}.
Vector quantizers, lattice codes, trellises, and palettes replace the scalar code set by a block codebook.
Their candidate-side moment summary is simply $\rmV_g(\alpha)\approx \E\left[\|\hW_g(\alpha)-\W_g(\alpha)\|_2^2\right]$ \citep{lee2025qpalette,chee2023quip,tseng2024quipsharp,liu2024vptq,tseng2025qtip}.

\textbf{Equivalent transformations change the coordinate before quantization.}
Scaling, affine transforms, and rotations keep the floating-point layer unchanged, but make the finite-format problem easier:
\begin{equation}
    \X(\alpha)=\Pmat\X,\qquad
    \W(\alpha)=\W\Pmat^{-1},\qquad
    \W(\alpha)\X(\alpha)=\W\X \ .
\end{equation}

These transformations keep the floating-point target fixed while moving quantization difficulty between the weight coordinate and the input coordinate \citep{xiao2023smoothquant,lin2024awq,shao2023omniquant,ma2024affinequant,sun2024flatquant}.
Diagonal scaling $ \Pmat=\diag(s) $ changes channel ranges and input magnitudes, which is why activation-aware scaling can protect important channels.
Orthogonal rotations $ \Pmat=\rmR,\ \rmR\tp\rmR=\I $ preserve Euclidean geometry while spreading outliers or improving incoherence before scalar, lattice, or trellis quantization \citep{ashkboos2024quarot,liu2024spinquant,chee2023quip,tseng2024quipsharp,tseng2025qtip}.
Fast Hadamard-sign transforms $ \Pmat=\rmH\rmD $ are the cheap deployable version when a full learned rotation is too expensive.

\textbf{Deployment costs define the deployable candidate set.}
A candidate enters $\Acal$ only if the backend supports its format, metadata, and kernels.
Along with the candidate coordinate and replacement weight, the selector therefore receives the deployment cost of the finite-format choice or transformation being scored:
\begin{equation}
    \Cost(\alpha)
    =
    \Cost_{\ntxt{bits}}(\alpha)
    +\Cost_{\ntxt{meta}}(\alpha)
    +\Cost_{\ntxt{kernel}}(\alpha)\ .
\end{equation}
Bits and fractional-bit palettes mainly affect storage and arithmetic cost.
Groups and codebooks add metadata or lookup overhead, and transformations may be folded offline or require online kernels \citep{yao2021hawqv3,lee2025qpalette,lin2025qserve}.
The cost does not measure accuracy degradation and only records which candidates are deployable under the budget.
What remains missing is the common output-distribution price that compares these deployable candidates before the completed quantized model exists.

\section{Priced Layer-Output Error for Configuration Selection}
\label{sec:priced-error}

The formulation specifies the constrained PTQ decision, but not a deployable score for ranking configurations before the completed model exists.
The target drift in \eqref{eq:target-drift} already names the right output-distribution quantity, and the missing step is to turn that model-level drift into a local calibration-time price for candidate layer configurations.
The forward-KL expansion supplies this price because its linear term vanishes at the full-precision reference.

\subsection{Forward KL Induces the Layer-Output Error Price}
\label{sec:forward-kl}

The target $\Jcal(\alpha)$ is the right fidelity measure, but evaluating it already requires a completed quantized model.
The pre-deployment score must therefore be the price that the full-precision reference assigns locally to the error injected by a candidate.
Using the candidate coordinate introduced in the setup, quantization produces $\hW_l(\alpha_l)=\W_l(\alpha_l)+\Dmat_l(\alpha_l)$ and injects $\Dmat_l(\alpha_l)\vx_l(\alpha_l)$ at the reference layer output $\vz_l=\W_l\vx_l$.
Set the input covariance $\Cxl(\alpha_l)\triangleq \E_{\vx\sim\Dcal}[\vx_l(\alpha_l)\vx_l(\alpha_l)\tp]$.
For a calibration input $\vx$, write $\Jcal_{\vx,l}(\vz)$ for the per-sample forward-KL drift viewed as a function of the layer output $\vz_l$, with the rest of the network held at the full-precision reference.
The two layer statistics that price the perturbation are
\begin{equation}
    \begin{aligned}
    \Hz_l
    &\triangleq
    \E_{\vx\sim\Dcal}\left[
    \left.\nabla_{\vz}^2\Jcal_{\vx,l}(\vz)\right|_{\vz=\vz_l}\right]\ ,\\
    \Cerr_l(\alpha_l)
    &\triangleq
    \Dmat_l(\alpha_l)\Cxl(\alpha_l)\Dmat_l(\alpha_l)\tp \ .
    \end{aligned}
\end{equation}
The first is the output-side curvature of the target drift, and the second is the candidate-induced layer-output error covariance.
A PSD Gauss-Newton (GN) surrogate may replace $\Hz_l$ in implementation, but the formal object being priced is still the covariance generated by the candidate weight perturbation.

\Needspace{10\baselineskip}
\begin{theorem}[Forward KL induces a quadratic quantization price]
\label{thm:price-predictor}
Assume the forward-KL objective is three-times differentiable on the weight segment from the full-precision reference to the candidate quantized model.
At the reference, the first-order term vanishes.
Keeping the block-diagonal second-order layer terms defined above, any candidate configuration $\alpha$ satisfies:
\begin{equation}
\label{eq:realized-configuration-price}
    \Jcal(\alpha)=\Price(\alpha)+r(\alpha)\ ,
    \qquad
    \Price(\alpha)\triangleq\sum_{l=1}^L \Price_l(\alpha_l)\ ,
    \qquad
    \Price_l(\alpha_l)\triangleq
    \frac12\Tr\left(\Hz_l\Cerr_l(\alpha_l)\right)\ .
\end{equation}
Stack $\vdelta_l\triangleq\operatorname{vec}(\Dmat_l(\alpha_l))$ into $\vdelta$.
Let $\Hz_{\theta}$ be the exact Hessian of $\Jcal$ with respect to these stacked layer weights, and let $\tilde{\Hz}_{\theta}$ be the retained block-diagonal Hessian satisfying $\Price(\alpha)=\frac12\vdelta\tp\tilde{\Hz}_{\theta}\vdelta$.
If $M(\alpha)$ bounds the third derivative along the perturbation segment and $\Price(\alpha)>0$, then
\begin{equation}
\label{eq:price-remainder-bound}
    \abs{r(\alpha)}\le \eta(\alpha)\Price(\alpha)\ ,
    \qquad
    \eta(\alpha)\triangleq
    \frac{
    \frac12\norm{\Hz_{\theta}-\tilde{\Hz}_{\theta}}_2\norm{\vdelta}_2^2
    +\frac{M(\alpha)}{6}\norm{\vdelta}_2^3
    }{\Price(\alpha)} \ .
\end{equation}
\end{theorem}

% \textbf{Remark.}
The proof is in \cref{app:proof-price-predictor}.
The theorem turns the configuration into a priced covariance, where the candidate supplies $\Cerr_l(\alpha_l)$ and the reference model supplies the curvature $\Hz_l$ that makes some output directions expensive.
The relative diagnostic $\eta(\alpha)$ has only one role, explaining why the omitted term is negligible when the retained Hessian is accurate along the quantization direction and the perturbation remains local.
In this regime the selector optimizes the quadratic price itself, and the factors of that price determine how PTQ configurations can be compared.

\begin{remark}[Reconstruction and diagonal scores are price reductions]
\label{rem:proxy-reductions}
Fix a layer and suppress $l$ and the candidate argument.
Let $\X=[\vx^{(1)},\ldots,\vx^{(N)}]$, $\Cx=N^{-1}\X\X\tp$, and $\Dmat_{\W}=\hW-\W$.
The full layer price $\Price=\frac12\Tr(\Hz\Dmat_{\W}\Cx\Dmat_{\W}\tp)$ becomes the layer reconstruction score by the substitution $\Hz\mapsto\I$, giving $\Price_{\ntxt{rec}}=\frac12\Tr(\Dmat_{\W}\Cx\Dmat_{\W}\tp) =\frac{1}{2N}\norm{\Dmat_{\W}\X}_F^2$.
A diagonal input-statistic score makes one further substitution, $\Cx\mapsto\diag(\Cx)$, giving $\Price_{\ntxt{diag}}=\frac12\sum_j[\Cx]_{jj}\norm{[\Dmat_{\W}]_{:,j}}_2^2$.
Thus reconstruction drops downstream output curvature, while diagonal scores also drop cross-channel input covariance.
These reductions preserve candidate rankings only when the dropped factors are effectively constant over the candidate set.
Otherwise reliable selection keeps the factors of the full price explicit.
\end{remark}

\subsection{From Candidate Families to Price Factors}
\label{sec:price-factor-perspectives}

The price theorem reduces configuration selection to one missing ingredient.
For every candidate available before deployment, the selector needs the layer-output error covariance that the candidate is expected to induce.
The candidate families in \Cref{sec:ptq-configuration-families} expose exactly the objects needed for that translation, including a candidate coordinate $(\W_l(\alpha_l),\X_l(\alpha_l))$, a quantized replacement $\hW_l(\alpha_l)$ or perturbation moment envelope $\rmV_l(\alpha_l)$, and a deployment cost $\Cost_l(\alpha_l)$.
The reference-side curvature $\Hz_l$ is not changed by these choices.
Finite formats change the perturbation moment, equivalent transformations change the coordinate in which that perturbation is generated, and the cost is carried forward to the search constraint.
These two configuration classes therefore explain how candidates become covariances, and the scoring reduction is a computational specialization rather than an additional candidate family.

\textbf{Configuration class for finite formats, bits, and granularity.}
This class fixes the finite values a weight may take.
Bitwidth, per-tensor, per-channel or group-wise sharing, scalar and microscaling formats, vector formats, and codebooks all determine the perturbation moment created by the replacement weight.
Before the completed model is built, the candidate can therefore provide the moment envelope $\rmV_l(\alpha_l)$ introduced in \cref{sec:ptq-configuration-families}.
Under the usual zero-mean independent-entry rounding envelope, with $\vh_l$ and $\vc_l$ collecting the diagonals of $\hHz_l$ and $\Cxl(\alpha_l)$, the expected price obeys:
\begin{equation}
\label{eq:finite-format-bit-bound}
    \E[\hPrice_l(\alpha_l)]
    \le
    \frac12\,\vh_l\tp\rmV_l(\alpha_l)\vc_l \ .
\end{equation}
\cref{app:proof-finite-format-price-bound} derives the bound from the scalar rounding model.
The display shows where the familiar bit law enters.
For a per-layer uniform range $r_l$, $\delta_b=2r_l/\Levels_b$ and $\rmV_l(b)=\delta_b^2\mathbf{1}\mathbf{1}\tp/12$, which gives
\begin{equation}
    \E[\hPrice_l(b)]
    \le
    \frac{r_l^2}{6\Levels_b^2}\Tr(\hHz_l)\Tr(\Cxl)\ ,
    \qquad \Levels_b=2^b-1 \ .
\end{equation}
Thus the $4^{-b}$ decay is not a standalone allocation law.
It appears only after the format has been converted into the perturbation moment being priced.
Granularity changes which entries share a range or step size, while vector and codebook quantizers use the same price slot by replacing scalar variances with empirical block perturbation moments.

\textbf{Configuration class for equivalent transformations.}
This class contains pre-quantization transformations that preserve the floating-point layer while changing the coordinate in which finite-format error is generated.
For an invertible transformation $\tx=\Pmat\vx$ and $\tW=\W\Pmat^{-1}$, the exact floating-point output is unchanged, but the quantized residual $\tDelta_b(\Pmat)\triangleq Q_b(\W\Pmat^{-1})-\W\Pmat^{-1}$ is generated on the transformed weight and is priced against the transformed input covariance.
An optimized transformation family $\mathcal P_l$ therefore contributes the candidate covariance through
\begin{equation}
\label{eq:weight-candidate-moment-price}
\begin{aligned}
    \hCerr_l(\Pmat,b)
    &=
    \E\left[
    \tDelta_b(\Pmat)(\Pmat\Cx\Pmat\tp)\tDelta_b(\Pmat)\tp
    \right]\ ,\\
    \Pmat^\star
    &\in
    \argmin_{\Pmat\in\mathcal P_l}
    \frac12\Tr\left(\hHz_l\hCerr_l(\Pmat,b)\right)\ .
\end{aligned}
\end{equation}
This display is the transformation analogue of the finite-format moment bound.
The matrix $\hHz_l$ remains a fixed reference-side price, while $\Pmat\Cx\Pmat\tp$ and $\tDelta_b(\Pmat)$ are the candidate-side factors a transformation can trade against each other.
The affine case isolates the clean input-metric effect before the finite-format residual is reintroduced.

\begin{theorem}[Affine preprocessing has a whitening optimum]
\label{thm:optimal-affine-input-metric}
Let $\Cx\succ0$ with dimension $n$ and set $\lambda_g=(\det\Cx)^{1/n}$.
Among determinant-preserving invertible input transformations,
\begin{equation}
    \min_{\Pmat\in\GL(n):\,\abs{\det\Pmat}=1}
    \lambda_{\max}(\Pmat\Cx\Pmat\tp)
    =
    \lambda_g \ .
\end{equation}
Equality holds if and only if $\Pmat\Cx\Pmat\tp=\lambda_g\I$.
Equivalently, if $\Cx=\rmU\Lambda\rmU\tp$, then $\Pmat=\lambda_g^{1/2}\Qmat\Lambda^{-1/2}\rmU\tp$ for some $\Qmat\in\Orth(n)$.
\end{theorem}

The theorem says only what a fully optimized affine transform can do to the input metric.
Under a determinant normalization, the smallest possible worst-direction variance is achieved by whitening.
If the quantizer residual were fixed and isotropic, this would reduce the candidate-side covariance to a scalar multiple of the identity and leave only the output-curvature trace.
In weight PTQ, however, $\Pmat$ also changes $\W\Pmat^{-1}$ and therefore the finite-format residual.
Practical transformations consequently optimize a restricted version of \eqref{eq:weight-candidate-moment-price}, balancing input-metric flattening against transformed-weight distortion.
\cref{app:proof-optimal-affine-input-metric} proves the whitening claim, and \cref{app:proof-configuration-invariance} records the covariance identity used by the transformed price.

Restricted transformation families inherit the same price object with less freedom in $\Pmat$.
Diagonal scaling $\Smat=\diag(s)$ gives
\begin{equation}
\label{eq:scaling-tradeoff}
    \hPrice_{\Smat,b}
    =
    \frac12\Tr\left(
    \hHz\,
    \E[\tDelta_b(\Smat)(\Smat\Cx\Smat)\tDelta_b(\Smat)\tp]
    \right)\ ,
    \quad
    \tDelta_b(\Smat)=Q_b(\W\Smat^{-1})-\W\Smat^{-1}\ .
\end{equation}
Scaling can lower price either by reducing the residual of $\W\Smat^{-1}$ or by reshaping variances in $\Smat\Cx\Smat$, but it cannot remove normalized input correlations (\cref{app:diagonal-correlation}).
Orthogonal rotations obey the same transformed-price form,
\begin{equation}
    \hPrice_{\rmR,b}
    =
    \frac12\Tr\left(
    \hHz\,
    \E[\tDelta_b(\rmR)(\rmR\Cx\rmR\tp)\tDelta_b(\rmR)\tp]
    \right)\ ,
    \quad
    \tDelta_b(\rmR)=Q_b(\W\rmR\tp)-\W\rmR\tp \ .
\end{equation}
Because $\rmR\Cx\rmR\tp$ preserves the eigenvalues of $\Cx$, rotations help through the finite-format residual, range, or coherence rather than through input whitening.
\cref{app:rotation-coherence} gives the corresponding scalar-quantizer distortion bound.
Transformations therefore change the candidate covariance, but the output metric that prices that covariance remains the reference-side $\Hz_l$.

\textbf{Scoring reduction for isotropic trace price.}
This reduction is not a configuration class.
It is the low-cost statistic used after a candidate family has made both the input metric and the perturbation moment close to isotropic.
The covariance-level formulas remain the definition of price, but an approximately isotropic candidate lets the selector avoid forming full matrices.
If $\Cx=\rmU\Lambda\rmU\tp$ and $\Pmat=\sigma\Qmat\Lambda^{-1/2}\rmU\tp$, then $\Pmat\Cx\Pmat\tp=\sigma^2\I$.
When the candidate perturbation moment is $\Rerr_l(\alpha_l)\triangleq
\E[\Dmat_l(\alpha_l)\Dmat_l(\alpha_l)\tp]
=\omega_l(\alpha_l)\I+\rmA_l(\alpha_l)$, the price reduces to
\begin{equation}
\label{eq:trace-reduction}
    \Price_l(\alpha_l)
    =
    \frac12\sigma_l^2(\alpha_l)\omega_l(\alpha_l)\Tr(\Hz_l)
    +
    \frac12\sigma_l^2(\alpha_l)
    \Tr\left(\Hz_l\rmA_l(\alpha_l)\right)\ .
\end{equation}
If the anisotropic residual $\rmA_l$ is small, the candidate is ranked by the single coefficient $\sigma_l^2(\alpha_l)\omega_l(\alpha_l)$ times the cached reference-side trace $\Tr(\Hz_l)$.
The reduction does not replace $\Hz_l$ by $\I$.
It keeps the output metric through its trace only after the candidate-side covariance has been made nearly isotropic.
\cref{app:proof-trace-reduction} gives the residual bound, and \cref{app:diagnostic-structures} records the diagnostics for regimes where diagonal curvature, off-diagonal input covariance, or imatrix-style reductions drop factors that are not constant over the candidate set.
Once the trace reduction is valid, each candidate can be scored by two candidate-side scalars and one reference-side curvature trace.

\subsection{Budgeted Price-Guided Configuration Search}
\label{sec:configuration-selection}

The trace reduction gives each candidate a computable search statistic, while $\Cost_l(\alpha_l)$ decides whether that scored covariance is deployable under the budget.
This statistic belongs to the search procedure, not to the configuration taxonomy.
For calibration samples $\{\vx^{(i)}\}_{i=1}^N$, input dimension $n_l$, and output dimension $m_l$, the implementation estimates the candidate-side input scale and perturbation scale as
\begin{equation}
    \hat\sigma_l^2(\alpha_l)\triangleq
    \frac{1}{n_l}\Tr\left(\frac1N\sum_{i=1}^N
    \vx_{l,i}(\alpha_l)(\vx_{l,i}(\alpha_l))\tp\right)\ ,
    \qquad
    \hat\omega_l(\alpha_l)\triangleq
    \frac{1}{m_l}\Tr\left(
    \Dmat_l(\alpha_l)\Dmat_l(\alpha_l)\tp\right)\ .
\end{equation}
When the candidate is represented by a moment envelope rather than an explicit replacement weight, the same normalized trace is taken from that envelope.
The reference model supplies $\hat\tau_l\triangleq\Tr(\hHz_l)$, estimated directly in large models by a PSD Gauss-Newton (GN) trace estimator such as Hutchinson \citep{hutchinson1990stochastic} rather than by forming $\hHz_l$.
The implemented score is therefore
\begin{equation}
\label{eq:implemented-trace-price}
    \hPrice_l(\alpha_l)
    \leftarrow
    \frac12
    \hat\sigma_l^2(\alpha_l)
    \hat\omega_l(\alpha_l)
    \hat\tau_l \ .
\end{equation}
The candidate supplies $\hat\sigma_l^2(\alpha_l)$, $\hat\omega_l(\alpha_l)$, and $\Cost_l(\alpha_l)$, while the reference model supplies the cached curvature trace $\hat\tau_l$.
Thus the expensive model-level trial is replaced by a calibration-time table of candidate prices and costs.

The remaining operation is the constrained selection itself.
As in training-free NAS \citep{shu2022nasi,shu2022hnas}, the selector scores many candidates before building their completed models.
The ideal but unavailable target in \eqref{eq:ideal-configuration-selection} is replaced by the predicted configuration price $\hPrice(\alpha)\triangleq\sum_{l=1}^L\hPrice_l(\alpha_l)$, while the deployment constraint is left unchanged:

\begin{equation}
\label{eq:configuration-selection}
    \alpha^\star
    \in \arg\min_{\alpha\in\Acal_1\times\cdots\times\Acal_L}
    \hPrice(\alpha)
    \quad \mathrm{s.t.}\quad
    \sum_{l=1}^L \Cost_l(\alpha_l)\le B \ .
\end{equation}
Equation~\eqref{eq:configuration-selection} is the pre-deployment replacement for exhaustive deployment trials.
The set $\Acal_l$ supplies admissible choices, $\Cost_l(\alpha_l)$ supplies deployment cost, and the score in \eqref{eq:implemented-trace-price} supplies the objective being minimized.
\cref{alg:price-guided-ptq} gives the resulting selector.

\vspace{-0.6em}
\begin{algorithm}[H]
\caption{Price-Guided Configuration Search}
\label{alg:price-guided-ptq}
\begin{algorithmic}[1]
\Require Reference model, calibration set, $\{(\Acal_l,\Cost_l)\}_{l=1}^L$, budget $B$.
\Ensure Selected configuration $\alpha^\star$.
\State Estimate and cache $\{\hat\tau_l=\Tr(\hHz_l)\}_{l=1}^L$ once from the full-precision forward-KL curvature.
\For{$l=1,\ldots,L$}
    \For{$\alpha_l\in\Acal_l$}
        \State Estimate candidate scalars $\hat\sigma_l^2(\alpha_l)$ and $\hat\omega_l(\alpha_l)$, then score
        $\hPrice_l(\alpha_l)\leftarrow\frac12\hat\sigma_l^2(\alpha_l)\hat\omega_l(\alpha_l)\hat\tau_l$.
    \EndFor
\EndFor
\State Select $\alpha^\star$ by solving the budgeted search problem in
\eqref{eq:configuration-selection}.
\State \Return $\alpha^\star$.
\end{algorithmic}
\end{algorithm}
\vspace{-2em}

\section{Experiments}
\label{sec:exp}
\label{sec:evidence}

The price-guided selector in \cref{alg:price-guided-ptq} is useful only if its calibration-time scores preserve the target decision before the quantized model is built.
The experiments therefore test the prediction chain in the same order as the method.
First, controlled perturbation studies ask whether the realized price in \eqref{eq:realized-configuration-price} and the implemented trace price in \eqref{eq:implemented-trace-price} track the full-precision-to-quantized KL drift (\cref{sec:exp-realized-price-kl}).
Second, end-to-end weight-only PTQ experiments ask whether optimizing the predicted price reduces configuration-search cost while preserving perplexity and downstream task accuracy (\cref{sec:exp-configuration-selection}).

\subsection{Realized Price and Implemented-Trace-Price Align with KL Drift}
\label{sec:exp-realized-price-kl}

The first empirical link is the one that can invalidate the selector.
If the realized layer-output price \eqref{eq:realized-configuration-price} and the implemented trace price \eqref{eq:implemented-trace-price} fail to track the final KL drift, then the price is not a valid surrogate for the target in \eqref{eq:target-drift}.
For each controlled candidate, we build the corresponding quantized perturbation, evaluate the full-precision-to-quantized KL drift on held-out calibration inputs, and compare it with the realized price and the implemented trace price.
We evaluate this comparison on OPT-125M \citep{opt} and Qwen3-0.6B \citep{qwen3}, across transformation, bitwidths, and granularity candidates. 
In particular, the transformation includes three types of whitening transformations (see \cref{app:whitening-normalizations}) and the no-transformation baseline.
We provide the detailed experimental setup in \cref{app:experiment-setup}.

\begin{figure}[htbp]
    \centering
    \vspace{-5mm}
    \includegraphics[width=\linewidth]{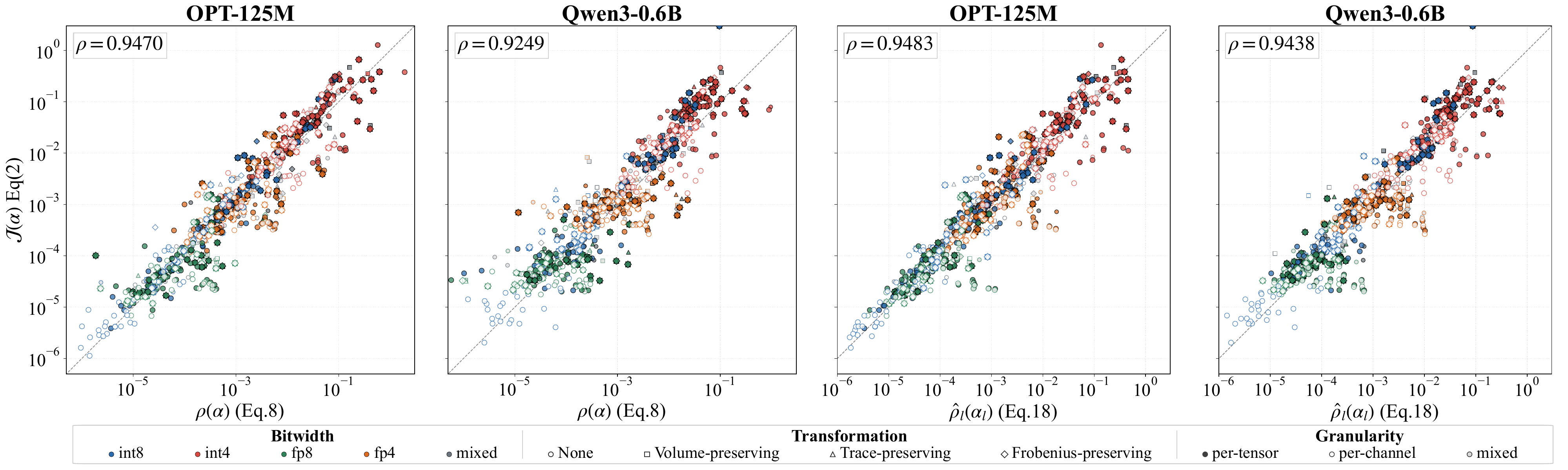}
    \caption{Realized quantization price \eqref{eq:realized-configuration-price} (left two panels) and implemented trace price \eqref{eq:implemented-trace-price} (right two panels) versus full-precision-to-quantized KL drift \eqref{eq:target-drift} under controlled perturbations. Each point is a candidate configuration, with marker styles indicating transformation constraints and colors indicating bitwidth or granularity choices. The reported $\rho$ summarizes the log-scale alignment.}
    \label{fig:realized-price-kl}
    \vspace{-5mm}
\end{figure}

\Cref{fig:realized-price-kl} shows that the realized price closely tracks the measured KL drift, with correlations $\rho=0.9470$ on OPT-125M and $\rho=0.9249$ on Qwen3-0.6B.
This is the expected behavior of the price theorem in \cref{thm:price-predictor}.
Although the controlled perturbation is applied to a single layer, we also report a two-layer price, shown by smaller markers in \cref{fig:realized-price-kl}, which sums the price of the perturbed layer and a second layer with a different candidate perturbation.
The two-layer price also tracks the KL drift, supporting the local linearized decomposition in \cref{thm:price-predictor} and consistent with the Linearity Theorem's claim that price is a local property of the quantization direction rather than a global property of the completed quantized model \citep{malinovskii2024linearity}.

The implemented trace price is a stricter test because it removes the full covariance calculation and keeps only the candidate-side scales $\hat\sigma_l^2(\alpha_l)$, $\hat\omega_l(\alpha_l)$, and the cached reference-side trace $\hat\tau_l$.
Despite this reduction, the right two panels retain nearly the same alignment with KL drift, with $\rho=0.9483$ on OPT-125M and $\rho=0.9438$ on Qwen3-0.6B.
This supports the intended role of \eqref{eq:implemented-trace-price} as a low-cost pre-deployment statistic rather than a reconstruction score, since it preserves the ordering induced by the realized quantization price in this controlled regime.

\subsection{Price-Guided Configuration Selection}
\label{sec:exp-configuration-selection}

After the KL comparisons establish that price is a meaningful ranking signal, the remaining question is whether that signal improves the discrete PTQ decision made before deployment.
We evaluate this decision on Llama-3.2-1B across three representative configuration classes, including bit allocation, transformation selection, and granularity allocation (refer to \cref{app:experiment-setup} for details).

\begin{table}[htbp]
\centering
\caption{Price-guided configuration selection results on Llama-3.2-1B.
The FP16 row is the unquantized reference.
PPL is evaluated on WikiText2.
Selection time reports the core configuration decision procedure, which excludes the subsequent quantization and evaluation pass.
The best and second-best results are highlighted in \textbf{bold} and \underline{underline}, respectively.}
\label{tab:configuration-selection-results}
\footnotesize
\setlength{\tabcolsep}{4.8pt}
\renewcommand{\arraystretch}{1.08}
\definecolor{ConfigBandGold}{RGB}{247,232,205}
\definecolor{ConfigBandBlue}{RGB}{221,230,250}
\definecolor{ConfigBandGreen}{RGB}{218,237,222}
\definecolor{ConfigAvg}{RGB}{255,249,207}
\newcommand{\configavg}[1]{\begingroup\setlength{\fboxsep}{1.2pt}\colorbox{ConfigAvg}{\strut\makebox[3.0em][c]{#1}}\endgroup}
\newcommand{\configband}[2]{\multicolumn{9}{@{}c@{}}{\begingroup\setlength{\fboxsep}{3pt}\colorbox{#1}{\makebox[\dimexpr\linewidth-2\fboxsep][c]{\strut\textbf{\small{#2}}}}\endgroup}}
\newcommand{\confighead}[2]{\begin{tabular}[c]{@{}c@{}}\textbf{#1}\\\textbf{#2}\end{tabular}}
\newcommand{\configpm}[2]{\begin{tabular}[c]{@{}c@{}}#1\\[-1pt]$\pm$ #2\end{tabular}}
\begin{tabular}{@{}lcccccccc@{}}
\toprule
\textbf{Method} &
\textbf{Time (s) $\downarrow$} &
\textbf{PPL $\downarrow$} &
\textbf{BoolQ $\uparrow$} &
\textbf{TruthfulQA $\uparrow$} &
\textbf{PIQA $\uparrow$} &
\textbf{WinoGrande $\uparrow$} &
\textbf{WiC $\uparrow$} &
\configavg{\textbf{Avg. $\uparrow$}} \\
\midrule
FP16 & - & 9.75 & 63.82 & 38.48 & 54.46 & 60.14 & 44.83 & \configavg{52.35} \\
\midrule
\configband{ConfigBandGold}{Bit Allocation} \\
\midrule
HIGGS \citep{malinovskii2024linearity} & \underline{813} & 12.50 & \underline{58.23} & \underline{40.24} & \underline{49.67} & \textbf{56.91} & \underline{47.81} & \configavg{\underline{50.57}} \\
AMQ \citep{lee2025amq} & 11635 & \textbf{11.05} & 55.26 & 40.12 & \textbf{50.82} & \underline{56.04} & \textbf{48.43} & \configavg{50.13} \\
\cmidrule(){1-9}
Ours & \textbf{579} & \underline{12.22} & \textbf{62.42} & \textbf{41.85} & 49.51 & 55.72 & 46.71 & \configavg{\textbf{51.24}} \\
\midrule
\configband{ConfigBandBlue}{Transformation Selection} \\
\midrule
CALM-CKA \citep{zhang2026calm} & 703 & \underline{12.80} & \underline{50.92} & 37.59 & 49.51 & \underline{56.12} & \textbf{51.88} & \configavg{\underline{49.20}} \\
Random Selection & - & 22.68 & 48.00 & \textbf{39.53} & \underline{49.75} & 54.56 & \underline{49.84} & \configavg{48.34} \\
\cmidrule(){1-9}
Ours & \textbf{657} & \textbf{11.53} & \textbf{53.64} & \underline{38.67} & \textbf{49.78} & \textbf{57.30} & 49.69 & \configavg{\textbf{49.82}} \\
\midrule
\configband{ConfigBandGreen}{Granularity Allocation} \\
\midrule
Fixed Group-128 & - & \underline{12.71} & \textbf{63.49} & 35.96 & \underline{51.96} & \underline{56.20} & \underline{49.84} & \configavg{\underline{51.49}} \\
Random Selection & - & 23313 & 43.57 & \textbf{48.64} & 3.10 & 50.62 & \textbf{51.15} & \configavg{39.41} \\
\cmidrule(){1-9}
Ours & 127 & \textbf{11.39} & \underline{62.54} & \underline{36.90} & \textbf{52.99} & \textbf{58.01} & 49.37 & \configavg{\textbf{51.96}} \\
\bottomrule
\end{tabular}
\vspace{-5mm}
\end{table}

\cref{tab:configuration-selection-results} summarizes the results.
For bit allocation, each linear layer chooses a bitwidth from the admissible low-bit set $[2, 3, 4]$ under a target average of $3.0$ bits, and we compare against AMQ \citep{lee2025amq} and HIGGS \citep{malinovskii2024linearity} allocation.
AMQ obtains the lowest WikiText2 PPL in this run, but its allocation time is much larger than that of the price-guided selector, since AMQ is an search-based method that builds and evaluates the quantized models for each candidate allocation.
Meanwhile, the price-guided allocation gives the best average downstream task score among the three 3-bit methods.
Compared with HIGGS, the price-guided allocation is faster and stronger on both PPL and average task accuracy.
These results do not imply that price is the universal best proxy for every metric, but they show that optimizing the predicted price can find a competitive allocation at substantially lower selection cost.

For transformation selection, the quantizer is held fixed while the selector chooses among admissible pre-quantization transformations, with CALM-CKA \citep{zhang2026calm} and random choice as baselines.
For granularity allocation, the selector chooses the scale at which weights share quantization parameters, with fixed group-128 and random choice as baselines.
In both settings, the price-guided choice gives the best PPL and the best average task score.
This indicates that pricing the induced layer-output error can select effective coordinates and granularities rather than treating these choices as backend details.

\section{Related Work and Scope}
\label{sec:related}
\label{sec:limitations}

The closest bit-allocation neighbors differ mainly in what has already been fixed before allocation begins.
HAWQ, HAWQ-V2, and HAWQ-V3 use Hessian-aware signals for mixed precision \citep{dong2019hawq,dong2019hawqv2,yao2021hawqv3}.
HIGGS-style methods relate layerwise reconstruction error to perplexity increase and allocate nonuniform quantization levels \citep{malinovskii2024linearity}.
Q-Palette takes a rate-distortion view of fractional-bit quantizers after Gaussianizing assumptions \citep{lee2025qpalette}.
These works are important fixed-geometry allocation references.
The present paper makes the geometry itself part of the pre-deployment configuration set and prices transformations, granularities, quantizer families, and bits with the same downstream metric.

Transformation and outlier methods occupy the other side of the same decision.
SmoothQuant, AWQ, OmniQuant, AffineQuant, FlatQuant, QuaRot, SpinQuant, QuIP, and QuIP\# show how scaling, rotations, learned affine transforms, incoherence, and outlier handling change quantization behavior \citep{xiao2023smoothquant,lin2024awq,shao2023omniquant,ma2024affinequant,sun2024flatquant,ashkboos2024quarot,liu2024spinquant,chee2023quip,tseng2024quipsharp}.
In this formulation, these transformations become comparable before deployment through the priced layer-output error they induce.
End-loss methods such as GuidedQuant and YAQA are closest in metric because they use downstream loss or output-distribution information \citep{guidedquant2025,tseng2025mpadaptive}.
This places the paper at the configuration-allocation level over admissible weight-space choices, with adaptive rounding objectives serving as close metric-aware neighbors.

The primary selector is deliberately scoped to weight-space PTQ.
Activation quantization, KV-cache quantization, runtime dynamic clipping, and QAT-style learned quantizers are inside the same price language only when their effect is exposed as a pre-deployment output-error covariance and a deployment cost.
The weight-activation extension, the AM-GM balance law inside the whitened isotropic surrogate, and the continuous allocation law are recorded formally in \cref{app:joint-quantization-extension} as branch consequences, not as the main selector's objective.
The local quadratic price is intended for regimes where low-bit nonlocality, moment-model mismatch, cross-layer coupling, and curvature-proxy error remain secondary to the predicted price gap.
Because the selector optimizes the price-only objective $\hPrice(\alpha)$, these effects define operating regimes for the predictor rather than extra optimization terms.

\section{Conclusion}

This paper formulates weight-space PTQ as pre-deployment configuration selection.
Before the completed quantized model reveals its output-distribution drift, each admissible layer choice must be ranked by the layer-output error it will induce and the deployment cost it will pay.
Starting from full-precision-to-quantized forward KL, the price theorem shows that score cancellation leaves a quadratic quantization price, so the reference model supplies the downstream curvature and each candidate supplies its error covariance.
This turns finite formats, granularities, codebooks, equivalent transformations, and fixed-geometry bit allocation into comparable instances of the same price-guided search, with the trace reduction providing a practical calibration-time score.
The resulting selector is most credible in the local, moment-stable, curvature-aligned regime tested by the evidence, and the same scope conditions identify where extending the price language beyond weight-space PTQ requires additional modeling.

\bibliographystyle{plainnat}
\bibliography{workspace/reference}

\appendix
\section{Extended Related Work}
\label{app:related}

\paragraph{Second-order and reconstruction-based PTQ.}
OBQ, GPTQ, and OBC-style methods use local second-order structure to compensate or order weight quantization decisions \citep{frantar2022gptq,frantar2023obc}.
QuIP and QuIP\# combine incoherence processing with quantization guarantees and codebook structure \citep{chee2023quip,tseng2024quipsharp}.
End-loss and model-preservation neighbors are the most direct novelty boundary because GuidedQuant uses end-loss guidance for quantization, while YAQA directly targets output-distribution preservation through end-to-end Hessian approximations and Kronecker sketches \citep{guidedquant2025,tseng2025mpadaptive}.
Analytical neighbors such as provable OPTQ, QERA, ASER, RUQuant, and activation-sensitivity analysis study local PTQ ingredients or reconstruction refinements \citep{zhang2025provableoptq,zhang2024qera,zhao2024aser,liu2026ruquant,xu2026activationsensitivity}.
The present paper uses these ingredients at a different decision level by predicting the price induced by candidate weight-space PTQ configurations and directly optimizing that price under deployment constraints.

\paragraph{Outliers, rotations, and affine transformations.}
LLM.int8(), Outlier Suppression, SmoothQuant, and AWQ show how calibration statistics and outlier ranges can motivate range-aware transformations \citep{dettmers2022llmint8,wei2022outlier,xiao2023smoothquant,lin2024awq}.
QuaRot, SpinQuant, AffineQuant, and FlatQuant motivate the transform-family view \citep{ashkboos2024quarot,liu2024spinquant,ma2024affinequant,sun2024flatquant}.
SpQR, SqueezeLLM, AQLM, GPTVQ, VPTQ, and QTIP expand the admissible weight perturbation set through sparse exceptions, dense-sparse decompositions, additive codebooks, vector quantization, or trellis-coded quantization \citep{dettmers2023spqr,kim2024squeezellm,egiazarian2024aqlm,vanbaalen2025gptvq,liu2024vptq,tseng2025qtip}.
In the pricing view used here, these methods matter only through the weight perturbation distribution they induce.
Activation quantization is a neighboring problem rather than part of the present model.

\paragraph{Mixed precision and deployment constraints.}
HAWQ, HAWQ-V2, and HAWQ-V3 use Hessian-aware signals for mixed-precision decisions \citep{dong2019hawq,dong2019hawqv2,yao2021hawqv3}.
SliM-LLM, Q-Palette, Linearity-Theorem-style allocation, PTQ1.61, QQQ, LRQ, Atom, KIVI, KVQuant, and QServe enlarge the deployment space through salience, fractional-bit formats, low-bit limits, vector or system-aware constraints, KV-cache compression, and serving co-design \citep{huang2024slimllm,lee2025qpalette,malinovskii2024linearity,zhao2025ptq161,zhang2024qqq,lee2025lrq,zhao2024atom,liu2024kivi,hooper2025kvquant,lin2025qserve}.
These methods sit downstream of the price because once candidate weight formats have predicted prices and deployment costs, hardware-aware deployment remains a discrete constrained selection problem.
The manuscript's allocation contribution is therefore the reduction from priced weight perturbations and transformation terms to coefficients such as $c_l$, not the claim that continuous KKT allocation is new.

\section{Proofs}
\label{app:additional-proofs-diagnostics}

\subsection{Proof of \texorpdfstring{\cref{thm:price-predictor}}{Thm.}}
\label{app:proof-price-predictor}
\begin{proof}
For each calibration input, the forward KL and the reference-supervised NLL differ by an entropy term independent of the candidate quantized model.
Write
\begin{equation}
    \Lref(\theta')\triangleq
    \E_{\vx\sim\Dcal}\left[
    -\sum_y\pizero(y\mid\vx)\log \pi_{\theta'}(y\mid\vx)\right]\ .
\end{equation}
At the full-precision reference,
\begin{equation}
    \nabla_{\theta'}\Lref(\theta')\big|_{\theta'=\theta}
    =
    -\E_{\vx\sim\Dcal}\left[\sum_y
    \pizero(y\mid\vx)\nabla_{\theta}\log\pizero(y\mid\vx)\right]
    =
    -\E_{\vx\sim\Dcal}\left[\sum_y\nabla_{\theta}\pizero(y\mid\vx)\right]=0 \ .
\end{equation}
Taylor expansion around the reference therefore has no first-order term.
For a linear layer represented in the candidate coordinate $\W_l(\alpha_l)\vx_l(\alpha_l)=\W_l\vx_l$, use the per-sample layer-output drift $\Jcal_{\vx,l}$ from the main text and let $\vw_l=\operatorname{vec}(\W_l(\alpha_l))$, $\vg_l=\left.\nabla_{\vz}\Jcal_{\vx,l}(\vz)\right|_{\vz=\vz_l}$, and $\Hz_l(\vx)=\left.\nabla_{\vz}^2\Jcal_{\vx,l}(\vz)\right|_{\vz=\vz_l}$.
The second-order chain rule gives
\begin{equation}
    \nabla_{\vw_l}^2\Jcal_{\vx,l}(\W_l(\alpha_l)\vx_l(\alpha_l))
    =
    \underbrace{(\I\otimes \vx_l(\alpha_l))\Hz_l(\vx)(\I\otimes \vx_l(\alpha_l))\tp}_{\ntxt{Gauss-Newton term}}
    +
    \underbrace{\sum_i [\vg_l]_i\nabla_{\vw_l}^2[\vz_l]_i}_{=0}\ ,
\end{equation}
because each coordinate of $\vz_l$ is linear in $\vw_l$.
Hence the retained block is $\Hz_l(\vx)\otimes \vx_l(\alpha_l)\vx_l(\alpha_l)\tp$ up to vectorization convention.
For $\vdelta_l=\operatorname{vec}(\Dmat_l(\alpha_l))$, the quadratic contribution is
\begin{equation}
    \frac12\E_{\vx\sim\Dcal}\left[
    \Tr\left(
    \Hz_l(\vx)
    \Dmat_l(\alpha_l)\vx_l(\alpha_l)\vx_l(\alpha_l)\tp
    \Dmat_l(\alpha_l)\tp
    \right)\right]\ .
\end{equation}
Replacing $\Hz_l(\vx)$ by the retained average $\Hz_l$ yields
\begin{equation}
    \frac12\Tr\left(
    \Hz_l
    \Dmat_l(\alpha_l)\Cxl(\alpha_l)\Dmat_l(\alpha_l)\tp
    \right)
    =
    \frac12\Tr\left(\Hz_l\Cerr_l(\alpha_l)\right)\ .
\end{equation}
Let $\Hz_{lk}$ be the exact Hessian block of $\Jcal$ with respect to vectorized layer weights at the reference, and let $\tilde{\Hz}_{ll}$ be the retained block satisfying $\frac12\vdelta_l\tp\tilde{\Hz}_{ll}\vdelta_l=\Price_l(\alpha_l)$.
For the stacked perturbation $\vdelta=(\vdelta_1,\ldots,\vdelta_L)$, Taylor's theorem gives
\begin{equation}
    \Jcal(\alpha)
    =
    \frac12\sum_l\vdelta_l\tp\Hz_{ll}\vdelta_l
    +
    \frac12\sum_{l\ne k}\vdelta_l\tp\Hz_{lk}\vdelta_k
    +
    R_3(\alpha)\ ,
    \qquad
    \abs{R_3(\alpha)}\le\frac{M(\alpha)}{6}\norm{\vdelta}_2^3 \ .
\end{equation}
The retained block-diagonal quadratic part is therefore
\begin{equation}
    \Price(\alpha)
    =
    \frac12\sum_l\vdelta_l\tp\tilde{\Hz}_{ll}\vdelta_l
    =
    \sum_l\frac12\Tr\left(\Hz_l\Cerr_l(\alpha_l)\right)\ ,
\end{equation}
which is \eqref{eq:realized-configuration-price}.
The omitted diagnostic term is
\begin{equation}
    r(\alpha)
    =
    \Jcal(\alpha)-\Price(\alpha)
    =
    \frac12\sum_l
    \vdelta_l\tp(\Hz_{ll}-\tilde{\Hz}_{ll})\vdelta_l
    +
    \frac12\sum_{l\ne k}\vdelta_l\tp\Hz_{lk}\vdelta_k
    +
    R_3(\alpha)\ .
\end{equation}
Equivalently, let $\Hz_{\theta}$ be the exact block Hessian with blocks $\Hz_{lk}$ and let $\tilde{\Hz}_{\theta}$ be block diagonal with blocks $\tilde{\Hz}_{ll}$.
Then $\Price(\alpha)=\frac12\vdelta\tp\tilde{\Hz}_{\theta}\vdelta$, so
\begin{equation}
    r(\alpha)
    =
    \frac12\vdelta\tp(\Hz_{\theta}-\tilde{\Hz}_{\theta})\vdelta
    +
    R_3(\alpha)\ .
\end{equation}
Cauchy's inequality and the Taylor remainder bound give the compact relative diagnostic used in the main text:
\begin{equation}
    \abs{r(\alpha)}
    \le
    \frac12\norm{\Hz_{\theta}-\tilde{\Hz}_{\theta}}_2\norm{\vdelta}_2^2
    +
    \frac{M(\alpha)}{6}\norm{\vdelta}_2^3 \ .
\end{equation}
When $\Price(\alpha)>0$, dividing this display by $\Price(\alpha)$ gives the relative diagnostic $\eta(\alpha)$ in \eqref{eq:price-remainder-bound}.
This proves the main-text price expression and identifies the proof-level terms omitted when the selector uses the quadratic price as its objective.
\end{proof}

\subsection{Finite-Format Moment Bound}
\label{app:proof-finite-format-price-bound}
\begin{proof}
Write $\Dmat=\Dmat_l(\alpha_l)$ and $\rmC=\Cxl(\alpha_l)$.
Let $\vh_l$ and $\vc_l$ collect the diagonals of $\hHz_l$ and $\rmC$, and let $\rmV_l(\alpha_l)$ be any nonnegative matrix satisfying $\E[\Dmat_{ij}^2]\le[\rmV_l(\alpha_l)]_{ij}$.
Expanding the covariance entrywise gives
\begin{equation}
    \E[(\Dmat\rmC\Dmat\tp)_{ii'}]
    =
    \sum_{j,k}[\rmC]_{jk}\E[\Dmat_{ij}\Dmat_{i'k}] \ .
\end{equation}
The independent zero-mean moment model removes all terms except $(i,j)=(i',k)$, hence
\begin{equation}
    \E[\hPrice_l(\alpha_l)]
    =
    \frac12\sum_i[\hHz_l]_{ii}\sum_j[\Cxl(\alpha_l)]_{jj}
    \E[\Dmat_{ij}^2]
    \le
    \frac12\vh_l\tp\rmV_l(\alpha_l)\vc_l\ .
\end{equation}
For the grouped scalar rounding model, $[\rmV_l(\alpha_l)]_{ij}=(\delta_{g(i,j)}(\alpha_l))^2/12$, which recovers the displayed bit and granularity bound in the main text.
\end{proof}

\subsection{Equivalent Transformation Covariance}
\label{app:proof-configuration-invariance}
\begin{proof}
Substituting $\tx=\Pmat\vx$ and $\tW=\W\Pmat^{-1}$ gives $\tW\tx=\W\vx$, so the floating-point output and its output-side curvature are unchanged.
Quantizing $\tW$ by $\tDelta_{\tW}$ creates output perturbation $\tDelta_{\tW}\tx$, whose covariance is $\tDelta_{\tW}\E[\tx\tx\tp]\tDelta_{\tW}\tp=\tDelta_{\tW}(\Pmat\Cx\Pmat\tp)\tDelta_{\tW}\tp$.
Vectorizing the quadratic form gives the Kronecker factor $\Hz\otimes(\Pmat\Cx\Pmat\tp)$, making explicit that the transformation moves the input-side metric and perturbation while leaving the output-side price metric fixed.
\end{proof}

\subsection{Proof of the Affine Whitening Optimum}
\label{app:proof-optimal-affine-input-metric}
\label{app:proof-whitened-trace-score}
\begin{proof}
Let $\Cx=\rmU\Lambda\rmU\tp$.
If $\Pmat=\sigma\Qmat\Lambda^{-1/2}\rmU\tp$ with $\Qmat\in\Orth(n)$, then
\begin{equation}
    \Pmat\Cx\Pmat\tp
    =
    \sigma^2\Qmat\Lambda^{-1/2}\Lambda\Lambda^{-1/2}\Qmat\tp
    =
    \sigma^2\I \ .
\end{equation}
Conversely, if $\Pmat\Cx\Pmat\tp=\sigma^2\I$, then $\Qmat=\sigma^{-1}\Pmat\rmU\Lambda^{1/2}$ satisfies $\Qmat\Qmat\tp=\I$, so $\Pmat=\sigma\Qmat\Lambda^{-1/2}\rmU\tp$.
Substituting the whitened covariance into the layer price gives
\begin{equation}
    \Price_{\Pmat}=\frac12\Tr(\Hz\tDelta_{\tW}(\sigma^2\I)\tDelta_{\tW}\tp)
    =
    \frac{\sigma^2}{2}\Tr(\Hz\tDelta_{\tW}\tDelta_{\tW}\tp)\ .
\end{equation}
If $|\det\Pmat|=1$, then $\det(\Pmat\Cx\Pmat\tp)=\det\Cx$.
For any positive definite matrix $\rmM$, $\lambda_{\max}(\rmM)\ge(\det\rmM)^{1/n}$ and $\kappa(\rmM)\ge1$.
Applying this to $\rmM=\Pmat\Cx\Pmat\tp$ gives the lower bounds.
Determinant-preserving whitening sets $\rmM=(\det\Cx)^{1/n}\I$, so it attains both lower bounds.
Equality in the $\lambda_{\max}$ lower bound holds only when all eigenvalues of $\rmM$ are equal, hence exactly when $\Pmat\Cx\Pmat\tp=(\det\Cx)^{1/n}\I$.
The preceding converse gives the displayed form of every optimizer.
Taking expectation under $\E[\tDelta_{\tW}\tDelta_{\tW}\tp]=\alpha\I+\rmA$ yields the stated decomposition with $\eps_{\ntxt{ani}}=(\sigma^2/2)\Tr(\Hz\rmA)$.
Holder's trace inequality gives $|\Tr(\Hz\rmA)|\le\norm{\Hz}_2\norm{\rmA}_*$.
\end{proof}

\subsection{Whitening Normalizations and Scale Coupling}
\label{app:whitening-normalizations}
The whitening family in \Cref{app:proof-whitened-trace-score} leaves a scalar degree of freedom.
Three standard normalizations fix it.
Let $\Cx=\rmU\Lambda\rmU\tp$ with eigenvalues $\lambda_1,\ldots,\lambda_n>0$, and set $\Pmat(\sigma,\Qmat)=\sigma\Qmat\Lambda^{-1/2}\rmU\tp$ with $\Qmat\in\Orth(n)$.
Then $\Pmat\Cx\Pmat\tp=\sigma^2\I$, and
\begin{equation}
    \sigma_{\ntxt{vol}}^2
    =
    \left(\prod_{i=1}^n\lambda_i\right)^{1/n}\ ,
    \qquad
    \sigma_{\ntxt{tr}}^2
    =
    \frac1n\sum_{i=1}^n\lambda_i\ ,
    \qquad
    \sigma_{\ntxt{F}}^2
    =
    \frac{n}{\sum_{i=1}^n\lambda_i^{-1}} \ .
\end{equation}
These are the geometric, arithmetic, and harmonic means induced by the constraints $|\det\Pmat|=1$, $\Tr(\Pmat\Cx\Pmat\tp)=\Tr(\Cx)$, and $\norm{\Pmat}_F^2=n$, respectively.
Under any of the three normalizations, the whitened weight-only price reduces to
\begin{equation}
    \Price_{\Pmat}
    =
    \frac{\sigma^2}{2}
    \Tr(\Hz\tDelta_{\tW}\tDelta_{\tW}\tp)\ .
\end{equation}
If the transformed-weight perturbation is isotropic in output space, $\E[\tDelta_{\tW}\tDelta_{\tW}\tp]=\alpha\I+\rmA$, this becomes
\begin{equation}
    \E[\Price_{\Pmat}]
    =
    \frac{\sigma^2\alpha}{2}\Tr(\Hz)
    +
    \frac{\sigma^2}{2}\Tr(\Hz\rmA)\ .
\end{equation}
The scalar $\sigma^2$ is therefore part of the candidate-side coefficient in the implemented statistic, not a replacement for output curvature.

A useful caveat is that, for adaptive scalar quantizers, the transformed-weight distortion can scale with the whitening normalization.
For a per-tensor uniform quantizer with range $R_q$ and $b$ bits,
\begin{equation}
    \E\left[\norm{\tDelta_{\tW}}_F^2\right]
    \approx
    mn\,\frac{R_q^2}{3\cdot 4^{b-1}}\ .
\end{equation}
If the range follows a homogeneous statistic of $\tW=\W\Pmat^{-1}=\sigma^{-1}\W\rmU\Lambda^{1/2}\Qmat\tp$, then $R_q$ can scale as $1/\sigma$, causing the explicit $\sigma^2$ in the input metric and the implicit $1/\sigma^2$ in distortion to partially cancel.
This is a quantizer-model-dependent diagnostic rather than a universal simplification.
The implemented selector computes $\hat\omega_l$ from the realized perturbation product or from the chosen finite-format moment model.

\subsection{Trace-Score Reduction}
\label{app:proof-trace-reduction}
\begin{proof}
Under the scalar input-metric condition $\Cxl(\alpha_l)=\sigma_l^2(\alpha_l)\I$, the layer-output error covariance induced by the candidate moment is
\begin{equation}
    \Cerr_l(\alpha_l)
    =
    \E\left[
    \Dmat_l(\alpha_l)\Cxl(\alpha_l)\Dmat_l(\alpha_l)\tp
    \right]
    =
    \sigma_l^2(\alpha_l)\Rerr_l(\alpha_l)\ .
\end{equation}
Substituting $\Rerr_l(\alpha_l)=\omega_l(\alpha_l)\I+\rmA_l(\alpha_l)$ into the layer price gives
\begin{equation}
    \Price_l(\alpha_l)
    =
    \frac12\sigma_l^2(\alpha_l)
    \Tr\left(\Hz_l(\omega_l(\alpha_l)\I+\rmA_l(\alpha_l))\right)\ ,
\end{equation}
which is \eqref{eq:trace-reduction} by linearity of the trace.
The second term is the residual left by non-isotropic candidate noise.
When a bound is needed, trace duality gives
\begin{equation}
    \left|
    \frac12\sigma_l^2(\alpha_l)
    \Tr\left(\Hz_l\rmA_l(\alpha_l)\right)
    \right|
    \le
    \frac12\sigma_l^2(\alpha_l)
    \norm{\Hz_l}_2\norm{\rmA_l(\alpha_l)}_* \ .
\end{equation}
\end{proof}

\subsection{Diagonal Scaling Shows What Scaling Cannot Change}
\label{app:diagonal-correlation}
For $\Smat=\diag(s_1,\ldots,s_n)$, define the correlation matrix of $\Cx$ by
\begin{equation}
    [\rmR_{\vx}]_{jk}=\frac{[\Cx]_{jk}}{\sqrt{[\Cx]_{jj}[\Cx]_{kk}}}\ .
\end{equation}
The transformed covariance is $\Smat\Cx \Smat$, and its correlation matrix has entries
\begin{equation}
    \frac{s_js_k[\Cx]_{jk}}{\sqrt{s_j^2[\Cx]_{jj}\,s_k^2[\Cx]_{kk}}}=[\rmR_{\vx}]_{jk}\ .
\end{equation}
Thus diagonal scaling can change variances but not standardized correlations.
This is why it cannot remove off-diagonal dependence in $\Cx$.

\subsection{Rotation Coherence Bounds Scalar-Quantizer Distortion}
\label{app:rotation-coherence}
The main text uses rotations only as a distortion-changing configuration.
A simple bound makes that role explicit.
For a symmetric per-tensor uniform quantizer with $\Levels_b=2^b-1$, define the coherence of a nonzero matrix $\rmA\in\R^{m\times n}$ by
\begin{equation}
    \mu(\rmA)=\frac{mn\norm{\rmA}_\infty^2}{\norm{\rmA}_F^2}\ .
\end{equation}
If $\tW=\W\rmR\tp$ and scalar rounding produces an error $\tDelta_{\tW}$, then the standard range bound gives
\begin{equation}
    \norm{\tDelta_{\tW}}_F^2
    \le
    \frac{\mu(\W\rmR\tp)\norm{\W}_F^2}{\Levels_b^2}\ .
\end{equation}
Consequently,
\begin{equation}
\label{eq:rotation-bound}
    \Price_{\rmR}(\tDelta_{\tW})
    \le
    \frac12\norm{\Hz}_2\lambda_{\max}(\Cx)
    \frac{\mu(\W\rmR\tp)\norm{\W}_F^2}{\Levels_b^2}\ .
\end{equation}
The input covariance spectrum is unchanged by $\rmR$.
The bound can improve only through the scalar-quantizer distortion factor.

\subsection{Curvature and Covariance Structure Diagnostics}
\label{app:diagnostic-structures}
The supporting diagnostics record three structural observations that determine when the scalarized trace statistic is appropriate.
First, output curvature $\Hz_l$ can be diagonal-dominant without being close to a scalar identity.
Writing
\begin{equation}
    \Hz_l=\diag(\Hz_l)+\rmE_l\ ,
\end{equation}
the diagonal dominance of $\Hz_l$ controls the size of $\rmE_l$, but the entries of $\diag(\Hz_l)$ can still vary substantially.
Replacing $\Hz_l$ by $\I$ therefore removes output-direction prices.
The trace reduction in \Cref{eq:trace-reduction} instead uses $\Tr(\Hz_l)$ only after the candidate-side output-noise covariance has been reduced to an isotropic matrix.

Second, raw input covariance $\Cxl$ can have substantial off-diagonal mass.
The diagonal-input reduction
\begin{equation}
    \frac12\Tr(\Hz\Dmat_{\W}\diag(\Cx)\Dmat_{\W}\tp)
\end{equation}
is a proxy, not the full price.
The diagnostics measure whether $\Cx-\diag(\Cx)$ behaves like secondary noise for the tested quadratic form.
When it does not, the full covariance or a transformation-generated covariance $\Pmat\Cx\Pmat\tp$ must be priced.

Third, the imatrix or weighted-quantization proxy used in some deployment stacks is exactly the reduced case obtained by setting $\Hz=\I$ and $\Cx=\diag(v_1,\ldots,v_n)$:
\begin{equation}
    \Price_{\ntxt{imatrix}}(\Dmat_{\W})
    =
    \frac12\sum_{j=1}^n v_j\norm{[\Dmat_{\W}]_{:,j}}_2^2\ ,
    \qquad
    v_j=\E[[\vx]_j^2]\ .
\end{equation}
Thus imatrix-style weights preserve diagonal input importance but omit both output curvature and cross-channel input covariance.
This is why the main text treats them as reduced proxies after the full price has been defined.

\subsection{Weight-Activation Extension and Reduced Allocation Laws}
\label{app:joint-quantization-extension}
The main selector is weight-space, but the supporting analysis also contains a neighboring weight-activation branch.
This branch uses the same output-price operator and is recorded here so that the branch coverage is explicit without turning the main text into a joint-quantization paper.

Fix a layer $\vz=\W\vx$.
Let the weight perturbation be $\Dmat_{\W}$ and write the activation perturbation as $\boldsymbol{\epsilon}_{\vx}=\hat{\vx}-\vx$.
Ignoring the product $\Dmat_{\W}\boldsymbol{\epsilon}_{\vx}$, the output perturbation is
\begin{equation}
    \boldsymbol{\epsilon}_{\vz}
    \approx
    \Dmat_{\W}\vx+\W\boldsymbol{\epsilon}_{\vx}\ .
\end{equation}
The weight-side first-order term is structurally zero at the reference by the same score-cancellation identity used in \Cref{thm:price-predictor}.
The activation-side first-order term is different because it vanishes only under a noise model such as $\E[\boldsymbol{\epsilon}_{\vx}\mid \vx]\approx 0$.
Under this conditional zero-mean assumption, and when $\E[\vx\boldsymbol{\epsilon}_{\vx}\tp]\approx0$, the retained output-error covariance becomes
\begin{equation}
    \Nout_{\ntxt{WA}}
    =
    \Dmat_{\W}\Cx\Dmat_{\W}\tp
    +
    \W\rmC_{\boldsymbol{\epsilon}}\W\tp\ ,
    \qquad
    \rmC_{\boldsymbol{\epsilon}}
    \triangleq
    \E[\boldsymbol{\epsilon}_{\vx}\boldsymbol{\epsilon}_{\vx}\tp]\ ,
\end{equation}
and the corresponding joint price is
\begin{equation}
    \Price_{\ntxt{WA}}
    =
    \frac12\Tr(\Hz\Nout_{\ntxt{WA}})\ .
\end{equation}
If the conditional zero-mean assumption fails, the omitted first-order bias and the cross moment are not hidden in the selector.
They are validity diagnostics for activation quantization in low-bit or strongly clipped regimes.

Under an equivalent input transform $\tx=\Pmat\vx$, $\tW=\W\Pmat^{-1}$, the two retained terms become
\begin{equation}
    \Price_W(\Pmat)
    =
    \frac12\Tr(\Hz\tDelta_{\tW}(\Pmat\Cx\Pmat\tp)\tDelta_{\tW}\tp)\ ,
    \qquad
    \Price_A(\Pmat)
    =
    \frac12\Tr(\Hz\W\Pmat^{-1}\rmC_{\boldsymbol{\epsilon},\tx}\Pmat^{-\top}\W\tp)\ .
\end{equation}
Thus preprocessing has a dual effect in the joint branch.
It can make the weight-side input metric more isotropic, while simultaneously changing the activation-noise covariance and the propagation matrix for activation noise.
For scale-equivariant per-channel activation quantization, this activation-side effect can cancel in the simplified surrogate.
For shared-range or non-equivariant activation quantization, scaling can change activation price.

Inside the whitened isotropic surrogate, let $\Pmat\Cx\Pmat\tp=\sigma^2\I$, $\E[\tDelta_{\tW}\tDelta_{\tW}\tp]\approx\alpha\I$, and $\rmC_{\boldsymbol{\epsilon},\tx}\approx\beta\I$.
Since $\Pmat^{-1}\Pmat^{-\top}=\sigma^{-2}\Cx$ for the whitening family,
\begin{equation}
    \Price_{\ntxt{WA}}(\sigma^2)
    =
    \frac{\alpha\sigma^2}{2}\Tr(\Hz)
    +
    \frac{\beta}{2\sigma^2}\Gamma\ ,
    \qquad
    \Gamma\triangleq\Tr(\Hz\W\Cx\W\tp)\ .
\end{equation}
This scalar branch has the exact AM-GM minimizer
\begin{equation}
    \sigma_\star^2
    =
    \left(\frac{\beta\Gamma}{\alpha\Tr(\Hz)}\right)^{1/2}\ ,
    \qquad
    \Price_{\ntxt{WA}}(\sigma_\star^2)
    =
    \sqrt{\alpha\beta\Gamma\Tr(\Hz)}\ .
\end{equation}
The equality is exact only after the whitening, isotropic-noise, and conditional-zero-mean reductions have been imposed.

The continuous bit-allocation law is another reduced surrogate.
If a layer has retained scalar coefficient $c_l>0$ and continuous bit variable $b_l$ with reduced price $c_l2^{-b_l}$, then under the interior constraint $\sum_l b_l=B$,
\begin{equation}
    b_l^\star
    =
    \bar b+\log_2\frac{c_l}{\cgeom}\ ,
    \qquad
    \bar b=\frac{B}{L}\ ,
    \qquad
    \cgeom=\left(\prod_{l=1}^Lc_l\right)^{1/L}\ .
\end{equation}
The optimized retained prices are equal:
\begin{equation}
    c_l2^{-b_l^\star}=\cgeom2^{-\bar b}
    \quad\text{for all }l\ .
\end{equation}
For discrete hardware bitsets, unequal costs, KV-cache constraints, or active lower/upper bounds, this law is a warm start or diagnostic for the budgeted configuration search, not a replacement for the discrete selector in \eqref{eq:configuration-selection}.

\section{Price-Guided Configuration Selection on 3B and 8B Models}
\label{app:additional-configuration-selection}

The main experiment in \cref{sec:exp-configuration-selection} uses Llama-3.2-1B to test whether the calibration-time price can guide a discrete PTQ decision before deployment.
We use the same protocol on Llama-3.2-3B and Llama-3.1-8B to check whether the same selection rule behaves similarly as the model size changes.
The three blocks keep the same interpretation as in the main text.
Bit allocation compares with HIGGS and AMQ after fixing the GPTQ backend, transformation selection compares with CALM-CKA and random selection, and granularity allocation compares with fixed group-128 and random selection.

\begin{table}[htbp]
\centering
\caption{Price-guided configuration selection results on Llama-3.2-3B.
The FP16 row is the unquantized reference.
PPL is evaluated on WikiText2.
Selection time reports the core configuration decision procedure, which excludes the subsequent quantization and evaluation pass.
The best and second-best results are highlighted in \textbf{bold} and \underline{underline}, respectively.}
\label{tab:configuration-selection-3b-results}
\footnotesize
\setlength{\tabcolsep}{4.8pt}
\renewcommand{\arraystretch}{1.08}
\definecolor{ConfigBandGold}{RGB}{247,232,205}
\definecolor{ConfigBandBlue}{RGB}{221,230,250}
\definecolor{ConfigBandGreen}{RGB}{218,237,222}
\definecolor{ConfigAvg}{RGB}{255,249,207}
\newcommand{\configavg}[1]{\begingroup\setlength{\fboxsep}{1.2pt}\colorbox{ConfigAvg}{\strut\makebox[3.0em][c]{#1}}\endgroup}
\newcommand{\configband}[2]{\multicolumn{9}{@{}c@{}}{\begingroup\setlength{\fboxsep}{3pt}\colorbox{#1}{\makebox[\dimexpr\linewidth-2\fboxsep][c]{\strut\textbf{\small{#2}}}}\endgroup}}
\newcommand{\confighead}[2]{\begin{tabular}[c]{@{}c@{}}\textbf{#1}\\\textbf{#2}\end{tabular}}
\newcommand{\configpm}[2]{\begin{tabular}[c]{@{}c@{}}#1\\[-1pt]$\pm$ #2\end{tabular}}
\begin{tabular}{@{}lcccccccc@{}}
\toprule
\textbf{Method} &
\textbf{Time (s) $\downarrow$} &
\textbf{PPL $\downarrow$} &
\textbf{BoolQ $\uparrow$} &
\textbf{TruthfulQA $\uparrow$} &
\textbf{PIQA $\uparrow$} &
\textbf{WinoGrande $\uparrow$} &
\textbf{WiC $\uparrow$} &
\configavg{\textbf{Avg. $\uparrow$}} \\
\midrule
FP16 & - & 7.81 & 74.04 & 39.28 & 72.25 & 68.43 & 50.31 & \configavg{60.86} \\
\midrule
\configband{ConfigBandGold}{Bit Allocation} \\
\midrule
HIGGS \citep{malinovskii2024linearity} & \underline{3417} & 13.18 & 51.10 & \textbf{40.39} & 18.55 & 55.64 & \underline{49.69} & \configavg{43.08} \\
AMQ \citep{lee2025amq} & 16265 & \textbf{9.41} & \underline{57.03} & 36.58 & \textbf{60.34} & \textbf{63.85} & 48.75 & \configavg{\underline{53.31}} \\
\cmidrule(){1-9}
Ours & \textbf{2182} & \underline{11.23} & \textbf{61.99} & \underline{39.10} & \underline{54.52} & \underline{60.54} & \textbf{52.04} & \configavg{\textbf{53.64}} \\
\midrule
\configband{ConfigBandBlue}{Transformation Selection} \\
\midrule
CALM-CKA \citep{zhang2026calm} & \textbf{970} & \underline{8.81} & \underline{71.90} & \textbf{40.12} & \underline{67.79} & \underline{64.48} & \textbf{48.59} & \configavg{\underline{58.58}} \\
Random Selection & - & 10.68 & 67.45 & 38.95 & 63.98 & 63.27 & \underline{48.54} & \configavg{56.44} \\
\cmidrule(){1-9}
Ours & 1004 & \textbf{8.58} & \textbf{72.26} & \underline{39.75} & \textbf{70.18} & \textbf{65.43} & 47.18 & \configavg{\textbf{58.96}} \\
\midrule
\configband{ConfigBandGreen}{Granularity Allocation} \\
\midrule
Fixed Group-128 & - & \underline{8.81} & \underline{71.62} & \underline{39.99} & \underline{68.17} & \underline{65.59} & \underline{46.55} & \configavg{\textbf{58.38}} \\
Random Selection & - & 2504.78 & 44.33 & \textbf{47.33} & 6.93 & 50.38 & \textbf{49.27} & \configavg{39.65} \\
\cmidrule(){1-9}
Ours & 372 & \textbf{8.46} & \textbf{72.11} & 37.98 & \textbf{68.50} & \textbf{65.67} & 46.39 & \configavg{\underline{58.13}} \\
\bottomrule
\end{tabular}
\vspace{-5mm}
\end{table}

\begin{table}[htbp]
\centering
\caption{Price-guided configuration selection results on Llama-3.1-8B.
The FP16 row is the unquantized reference.
PPL is evaluated on WikiText2.
Selection time reports the core configuration decision procedure, which excludes the subsequent quantization and evaluation pass.
The best and second-best results are highlighted in \textbf{bold} and \underline{underline}, respectively.}
\label{tab:configuration-selection-8b-results}
\footnotesize
\setlength{\tabcolsep}{4.8pt}
\renewcommand{\arraystretch}{1.08}
\definecolor{ConfigBandGold}{RGB}{247,232,205}
\definecolor{ConfigBandBlue}{RGB}{221,230,250}
\definecolor{ConfigBandGreen}{RGB}{218,237,222}
\definecolor{ConfigAvg}{RGB}{255,249,207}
\newcommand{\configavg}[1]{\begingroup\setlength{\fboxsep}{1.2pt}\colorbox{ConfigAvg}{\strut\makebox[3.0em][c]{#1}}\endgroup}
\newcommand{\configband}[2]{\multicolumn{9}{@{}c@{}}{\begingroup\setlength{\fboxsep}{3pt}\colorbox{#1}{\makebox[\dimexpr\linewidth-2\fboxsep][c]{\strut\textbf{\small{#2}}}}\endgroup}}
\newcommand{\confighead}[2]{\begin{tabular}[c]{@{}c@{}}\textbf{#1}\\\textbf{#2}\end{tabular}}
\newcommand{\configpm}[2]{\begin{tabular}[c]{@{}c@{}}#1\\[-1pt]$\pm$ #2\end{tabular}}
\begin{tabular}{@{}lcccccccc@{}}
\toprule
\textbf{Method} &
\textbf{Time (s) $\downarrow$} &
\textbf{PPL $\downarrow$} &
\textbf{BoolQ $\uparrow$} &
\textbf{TruthfulQA $\uparrow$} &
\textbf{PIQA $\uparrow$} &
\textbf{WinoGrande $\uparrow$} &
\textbf{WiC $\uparrow$} &
\configavg{\textbf{Avg. $\uparrow$}} \\
\midrule
FP16 & - & 6.24 & 83.00 & 44.17 & 77.09 & 71.35 & 51.10 & \configavg{65.34} \\
\midrule
\configband{ConfigBandGold}{Bit Allocation} \\
\midrule
HIGGS \citep{malinovskii2024linearity} & \textbf{5805} & 11.15 & 58.99 & \underline{39.56} & 28.40 & 58.72 & 49.69 & \configavg{47.07} \\
AMQ \citep{lee2025amq} & 23337 & \textbf{9.31} & \textbf{63.73} & \textbf{41.70} & \textbf{55.88} & \textbf{62.90} & \textbf{50.00} & \configavg{\textbf{54.84}} \\
\cmidrule(){1-9}
Ours & \underline{5967} & \underline{10.87} & \underline{60.83} & 38.74 & \underline{36.67} & \underline{60.46} & \textbf{50.00} & \configavg{\underline{49.34}} \\
\midrule
\configband{ConfigBandBlue}{Transformation Selection} \\
\midrule
CALM-CKA \citep{zhang2026calm} & \textbf{1957} & \underline{6.97} & \textbf{81.38} & \underline{42.19} & \underline{76.61} & \textbf{70.72} & \underline{50.00} & \configavg{\underline{64.18}} \\
Random Selection & - & 24.34 & 65.06 & \textbf{42.66} & 34.69 & 61.19 & 48.64 & \configavg{50.45} \\
\cmidrule(){1-9}
Ours & 1988 & \textbf{6.86} & \underline{80.43} & 41.80 & \textbf{77.48} & \underline{70.56} & \textbf{50.94} & \configavg{\textbf{64.24}} \\
\midrule
\configband{ConfigBandGreen}{Granularity Allocation} \\
\midrule
Fixed Group-128 & - & \underline{7.02} & \underline{80.89} & 40.39 & \underline{71.76} & \underline{70.64} & \textbf{50.00} & \configavg{\underline{62.74}} \\
Random Selection & - & 2326.04 & 45.28 & \textbf{48.73} & 8.94 & 50.59 & 49.32 & \configavg{40.57} \\
\cmidrule(){1-9}
Ours & 675 & \textbf{6.75} & \textbf{81.47} & \underline{41.28} & \textbf{74.97} & \textbf{71.11} & \textbf{50.00} & \configavg{\textbf{63.77}} \\
\bottomrule
\end{tabular}
\vspace{-5mm}
\end{table}

\Cref{tab:configuration-selection-3b-results,tab:configuration-selection-8b-results} show a consistent pattern for transformation and granularity choices, where the selected configuration changes how the candidate perturbation induces layer-output error.
In transformation selection, the price-guided selector gives the best PPL and the best average task score on both models.
On Llama-3.2-3B it improves over CALM-CKA from 8.81 to 8.58 PPL and from 58.58 to 58.96 average score.
On Llama-3.1-8B it similarly improves PPL from 6.97 to 6.86 and gives a slightly higher average score, 64.24 versus 64.18, while maintaining the comparable selection time cost of CALM-CKA.
For granularity allocation, the price-guided choice gives the best PPL on both models and the best average score on Llama-3.1-8B.
The random baselines are much less reliable in these two blocks, especially in PPL, so the price-guided rule is not just avoiding a weak fixed default.

The bit-allocation rows are less uniform.
On Llama-3.2-3B, the price-guided allocation has the best average downstream score, 53.64, while AMQ has the best PPL.
On Llama-3.1-8B, AMQ is stronger across this block, and the price-guided row is second in average score among the three allocation methods.
However, it is worth noting that the AMQ method spend much more time than our price-guided selector, since AMQ is an search-based method that builds and evaluates the quantized models for each candidate allocation. 

\FloatBarrier

\section{Experiments Setup}
\label{app:experiment-setup}

\subsection{Baselines}
Here we describe the baselines used in the price-guided configuration selection experiments in \cref{sec:exp-configuration-selection} and \cref{app:additional-configuration-selection}.

\paragraph{Bit Allocation.}
\begin{itemize}[leftmargin=*]
    \item \textbf{HIGGS} \citep{malinovskii2024linearity}: HIGGS is a mixed-precision method based on Linearity Theorem. It links per-layer $\ell_2$ reconstruction error to perplexity increase through layerwise linear coefficients, and uses the resulting linearized objective to choose nonuniform layer bitwidths under the same compression budget.
    \item \textbf{AMQ} \citep{lee2025amq}: AMQ is a search-based mixed-precision weight-only quantization baseline. It formulates layerwise bit allocation as a discrete multi objective search over bitwidth assignments under a memory budget, then makes the search practical with an iterative NSGA-II method before deploying the selected allocation with the final quantization backend. It is worth noting that search-based methods like AMQ can cost much more time than our price-guided selector.
\end{itemize}

\paragraph{Transformation Selection.}
\begin{itemize}[leftmargin=*]
    \item \textbf{CALM-CKA} \citep{zhang2026calm}: CALM-CKA is a representation similarity baseline adapted from CALM, which evaluates candidate PTQ modules layer by layer and selects the option whose quantized representation best preserves the full-precision representation under linear CKA. In our transformation-selection experiment, we keep the quantizer, bitwidth, and backend fixed and use the same CKA score to choose among admissible pre-quantization transformations.
    \item \textbf{Random Choice}: We randomly sample the admissible transformation candidate from the same candidate set used by the price-guided selector and then quantize.
\end{itemize}

\paragraph{Granularity Selection.}
\begin{itemize}[leftmargin=*]
    \item \textbf{Fixed Group-128}: We use a fixed group size of $128$ for all quantized linear layers, keeping the quantizer and bitwidth fixed so that the comparison isolates the value of layerwise granularity selection.
    \item \textbf{Random Choice}: Same as in transformation selection, we randomly sample the admissible granularity candidate from the same candidate set used by the price-guided selector and then quantize.
\end{itemize}

\subsection{Hyperparameters and Implementation Details}

\paragraph{Common setup.}
The KL-alignment study in \cref{fig:realized-price-kl} uses OPT-125M and Qwen3-0.6B.
The configuration-selection experiments use Llama-3.2-1B in the main text and Llama-3.2-3B and Llama-3.1-8B in \cref{app:additional-configuration-selection}.
For the Llama models, we quantize all transformer linear projections except embeddings, normalization layers, rotary embeddings, and the language-model head, giving $112$, $196$, and $224$ quantized linear layers for the 1B, 3B, and 8B models.
Calibration statistics are collected on the WikiText2 validation split with $32$ samples, maximum sequence length $512$, and seed $0$.
The KL-alignment run uses the same dataset family with $128$ validation samples of length $1024$.

\paragraph{Price estimation and bit allocation.}
For Llama selection, the downstream curvature term is the Gauss--Newton Fisher trace estimate
$\hat\tau_l\triangleq\Tr(\hHz_l)=\mathbb{E}\|\nabla_{z_l}\ell\|_2^2$, computed by one forward and backward pass per layer; no Hessian-vector probes are used in the selector.
Transformation and granularity candidates are scored by
$\hPrice_l(\alpha_l)=\frac12\hat\tau_l\,\mathrm{MSE}(z_l,\hat z_l(\alpha_l))$
on at most $512$ supervised calibration tokens per layer.
For bit allocation, each layer chooses from $\{2,3,4\}$ bits under the AMQ-compatible budget $3.0\pm0.005$, with group size $128$.
The price-guided assignment is solved as a binary MILP; HIGGS and AMQ use the same bit set, group size, target budget, and GPTQ backend.
GPTQ quantization uses WikiText2 calibration data, true-sequential quantization, activation ordering, and no MSE search.

\paragraph{Candidate sets.}
Transformation selection fixes W4 quantization with group size $128$ and chooses among the identity transform, AWQ-style scaling, SmoothQuant-style scaling, Hadamard rotation, and normalized diagonal whitening.
The SmoothQuant-style scale uses $\alpha=0.5$; Hadamard rotations use deterministic layerwise Rademacher signs with seed $0$ when the input dimension is a power of two.
The CALM-CKA baseline chooses the candidate with maximum linear CKA, and random baselines use seeds $0$, $1$, and $2$.
Granularity selection fixes W4, no transformation, uniform integer quantization, and min--max clipping, and chooses among per-tensor, per-channel, group sizes $32$, $64$, and $128$, and $64\times64$ block-wise quantization.
Granularity costs include FP32 scale metadata, reported as effective weight bits; ties within $1\%$ of the best price are broken by lower effective bit cost and then by the latency proxy.
Transformation and granularity checkpoints are saved as dense fake-quant Hugging Face checkpoints rather than packed deployment artifacts.

\paragraph{Evaluation protocol.}
WikiText2 PPL is evaluated on the test split using full-text concatenation with slow tokenization, sequence length $2048$, $141$ sequences, and $288{,}627$ supervised tokens.
Downstream accuracy is evaluated with LightEval, batch size $32$, greedy decoding with temperature $0$ and seed $42$.
The five tasks are BoolQ \citep{clark-etal-2019-boolq}, TruthfulQA MC2 \citep{lin-etal-2022-truthfulqa}, PIQA \citep{bisk2020piqa}, WinoGrande \citep{sakaguchi2021winogrande}, and WiC \citep{pilehvar-camacho-collados-2019-wic}; the reported average is the arithmetic mean of the five percentages.
FP16 references use float16.

\end{document}